\documentclass[11pt]{article}
\usepackage{amsmath,amssymb,amsbsy,amsthm, fullpage}

\let\epsilon\varepsilon
\let\phi\varphi

\def\X{{\cal X}}
\def\N{\mathbb N}

\def\C{\mathcal C}
\def\E{{\bf E}}

\def\-as{\text{-a.s.}}

\newtheorem{theorem}{Theorem}
\newtheorem{definition}{Definition}

\begin{document}
\title{On Finding Predictors for Arbitrary Families of Processes}
\author{{\bf Daniil Ryabko}\\    daniil@ryabko.net,   INRIA Lille
}
\date{}
%\editor{}
\maketitle

\begin{abstract}%\footnote{Some preliminary results were reported at UAI'09 \cite{Ryabko:09pq3} and ALT'08 \cite{Ryabko:08pq+}.}
The problem is sequence prediction in the following setting. 
A sequence $x_1,\dots,x_n,\dots$ of discrete-valued observations is generated 
according to some unknown probabilistic law (measure) $\mu$. After observing each outcome, 
it is required to give the conditional probabilities of the next observation.
The measure  $\mu$ belongs to an arbitrary but known class $\C$  of stochastic process measures.
We are interested in predictors $\rho$ whose conditional probabilities converge (in some sense) to the
``true'' $\mu$-conditional probabilities if any $\mu\in\C$ is chosen to generate the sequence.
The contribution of this work is in characterizing the families $\C$ for which such predictors exist, 
and in providing a specific and simple form in which to look for a solution. We show that if any predictor works, then 
there exists a Bayesian predictor,  whose   prior is discrete, and which works too.
 We also find several sufficient and necessary conditions
for the existence of a predictor, in terms of topological characterizations of the family $\C$, as well as in terms
of local behaviour of the measures in $\C$,  which in some cases lead to procedures for constructing such predictors.

It should be  emphasized that the framework is completely general: the stochastic processes considered are not required to be 
i.i.d., stationary, or to belong to any parametric or countable family.
\end{abstract}

% \begin{keywords}
%   Sequence prediction,  Time series, Online prediction, Bayesian prediction
% \end{keywords}
\section{Introduction}
Given a  sequence $x_1,\dots,x_n$ of observations $x_i\in\X$, where $\X$ is a finite set, we 
want to predict what are the probabilities of observing $x_{n+1}=x$ for each $x\in\X$, or, more generally,
probabilities of observing different $x_{n+1},\dots,x_{n+h}$, before $x_{n+1}$ is revealed,
after which the process continues.
It is assumed that the sequence is generated by some unknown stochastic process $\mu$, a probability measure
on the space of one-way infinite sequences $\X^\infty$. The goal is to have a predictor whose predicted probabilities
converge (in a certain sense) to the correct ones (that is, to $\mu$-conditional probabilities). In general this goal is impossible to achieve if 
nothing is known about the measure $\mu$ generating the sequence. In other words, one cannot have a predictor
whose error goes to zero for any measure $\mu$. The problem becomes tractable if we assume that the measure $\mu$
generating the data belongs to some known class $\C$.
The  questions addressed in this work are a part of the following general problem: given an arbitrary set  $\C$ of measures, how can we find 
a predictor that performs well when the data is generated by any  $\mu\in\C$, and whether
it is possible to find such a predictor at all. 
An example of a generic  property  of a class $\C$ that allows for construction of a predictor, is that
$\C$ is countable. Clearly, this condition is very strong. An example,  important from the applications point of view,  of a class $\C$ of measures  
for which  predictors are known,  is the class of all stationary measures. The general question, however, is very far from being answered.

The contribution of this work to solving this question is, first, in that we 
 provide a specific form in which to look for a predictor.
More precisely, we show that if a predictor that predicts every $\mu\in\C$ exists, 
then such a predictor  can also be obtained as a weighted sum of  countably many elements of $\C$.  This result
can also be viewed as a justification of the Bayesian approach to sequence prediction: if there exists 
a predictor which predicts well every measure in the class, then there exists a Bayesian predictor (with a rather simple prior) that has this property too.
In this respect it is important to note that the result obtained about such a  Bayesian predictor
is pointwise (holds for every $\mu$ in $\C$), and stretches far beyond the set its prior is concentrated on.
Next, we derive some characterizations of families $\C$ for which a  predictor exist. We first
analyze what is furnished by the notion of separability, when a suitable topology can be found: we find that  it is a sufficient
but not always a necessary condition. We then derive some sufficient conditions
for the existence of a predictor which are based on local (truncated to the first $n$ observation) behaviour of measures
in the class $\C$. Necessary conditions cannot be obtained in this way (as we demonstrate), but  sufficient conditions, along with rates 
of convergence and construction of predictors, can be found.

The {\bf motivation} for studying predictors for arbitrary classes $\C$ of processes is two-fold. First of all, prediction is a basic ingredient for 
constructing intelligent systems. Indeed, in order to be able to find optimal behaviour in an unknown environment,
an intelligent agent must be able, at the very least, to predict how the environment is going to behave (or, to be more precise, how relevant 
parts of the environment are going to behave).
Since the response of the environment may in general depend on the actions of the agent, this response
is necessarily non-stationary for explorative agents. Therefore, one cannot readily use prediction methods developed for stationary 
environments, but rather has to find predictors for the classes of processes that can appear as a possible response
of the environment. 

Apart from this, the problem of prediction itself has numerous applications in such diverse
fields as data compression, market analysis,  bioinformatics, and many others. It seems clear that prediction methods constructed
for one application cannot be expected to be optimal when applied to another. Therefore, an important question is how to develop
specific prediction algorithms for each of the domains. 

{\bf Prior work}.
As it was mentioned,  if the class $\C$ of measures is countable (that is, if $\C$ can be represented as $\C:=\{\mu_k: k\in\N\}$), then 
there exists a predictor which performs well for any $\mu\in\C$. Such a predictor can be obtained as a  Bayesian mixture 
$\rho_S:=\sum_{k\in\N} w_k \mu_k$, where $w_k$ are summable positive real weights, and it has very strong predictive properties; in particular,  $\rho_S$ predicts every $\mu\in\C$ in total variation distance, as follows from the result of  \cite{Blackwell:62}.
Total variation distance measures the difference in (predicted and true) conditional probabilities of all future 
events, that is, not only the probabilities of the next observations, but also of observations that are arbitrary far off in the future (see formal 
definitions below).  In the context of sequence prediction the measure $\rho_S$  was first studied 
by  \cite{Solomonoff:78}. Since then, the idea of taking a convex combination of a finite or countable class of measures (or predictors) to
obtain a predictor permeates most of the research on sequential prediction (see, for example, \cite{Cesa:06}) and  more general learning problems~\cite{Hutter:04uaibook, Ryabko:08ao++}. 
In practice it is clear that, on the one hand, countable models are not sufficient, since already the class $\mu_p, p\in[0,1]$ of
Bernoulli i.i.d. processes, where $p$ is the probability of 0, is not countable. On the other hand, prediction in total variation
can be too strong to require; predicting probabilities of the next observation may be sufficient, maybe even not 
on every step but in the Cesaro sense. A key observation here is that a predictor $\rho_S=\sum w_k\mu_k$ may be a good predictor
not only when the data is generated by one of the processes $\mu_k$, $k\in\N$, but when it comes from a much larger class.
Let us consider this point in more detail. Fix for simplicity $\X=\{0,1\}$. The Laplace predictor
\begin{equation}\label{eq:lapl}
  \lambda(x_{n+1}=0|x_1,\dots,x_n)=\frac{\#\{i\le n: x_i=0\}+1}{n+|\X|}
\end{equation}
predicts any Bernoulli i.i.d.~process: although convergence in total variation distance of conditional probabilities
does not hold, predicted probabilities of the next outcome converge to the correct ones.
 Moreover, generalizing the Laplace predictor,  a predictor  $\lambda_k$ can be constructed for 
the class $M_k$ of all $k$-order Markov measures, for any given $k$. As was found by \cite{BRyabko:88}, the combination $\rho_R:=\sum w_k\lambda_k$
is a good predictor not only for the set $\cup_{k\in\N} M_k$ of all finite-memory processes, but also for any measure
$\mu$ coming from a much larger class: that of all stationary measures on $\X^\infty$. Here prediction is possible
only in the Cesaro sense (more precisely, $\rho_R$ predicts every stationary process in expected time-average  Kullback-Leibler divergence,
see definitions below).
The Laplace predictor itself can be obtained as a Bayes mixture over all Bernoulli i.i.d. measures with uniform 
prior on the parameter $p$ (the probability of 0). However, as was observed in \cite{Hutter:07upb} (and as is easy to see),
the same (asymptotic) predictive properties are possessed by  a Bayes mixture with a countably supported prior which is dense in 
$[0,1]$ (e.g. taking $\rho:=\sum w_k \delta_k$ where $\delta_k, k\in\N$ ranges over all Bernoulli i.i.d. measures with rational probability of 0). 
For a given $k$, the set of $k$-order Markov processes is parametrized by finitely many $[0,1]$-valued parameters. Taking a dense 
subset of the values of these parameters, and a mixture of the corresponding measures, results in a predictor for the class
of $k$-order Markov processes. Mixing over these (for all $k\in\N$) yields, as in \cite{BRyabko:88}, a predictor for the class of all 
stationary processes. Thus, for the mentioned classes of processes, a predictor can be obtained as a Bayes mixture of 
 countably many measures in the class. An additional reason why this kind  of analysis is interesting is because of the difficulties arising in trying to construct 
Bayesian predictors for classes of processes that can not be easily parametrized. Indeed, a natural way to obtain 
a predictor for a class $\C$ of stochastic processes is to take a Bayesian mixture of the class. 
To do this, one needs to define the structure of a probability space on $\C$. 
If the class $\C$ is well parametrized, as is the case with the set of all Bernoulli i.i.d. process, then one can 
integrate with respect to the parametrization. In general, when the problem lacks a natural parametrization, although one can define the structure of the probability 
space on the set of (all) stochastic process measures in many different ways, the results one can obtain will then be
with probability 1 with respect to the prior distribution (see, for example, \cite{Jackson:99}). Pointwise consistency
cannot be assured (see e.g. \cite{Diaconis:86}) in this case, meaning that some  (well-defined) Bayesian predictors are not consistent
on some (large) subset of $\C$. 
Results with prior probability 1  can be hard to interpret if one is not sure that the structure 
of the probability space defined on the set $\C$ is indeed a natural one for the problem at hand (whereas if one does have a natural parametrization,
then usually results for every value of the parameter can be obtained, as in the case with Bernoulli i.i.d. processes mentioned above).
The results of the present work show that when a predictor exists it can indeed be given as  a Bayesian 
predictor, which predicts  every (and not almost every) measure in the class, while its support is only a countable set.

A related question is formulated as a question about
two individual measures, rather than about a class of measures and a predictor.
Namely, one can ask under which conditions one stochastic process 
predicts another. In \cite{Blackwell:62} it was shown that 
if one measure is absolutely continuous with respect to another, than 
the latter predicts the former (the conditional probabilities converge in a very strong sense).
In \cite{Ryabko:07pqisit, Ryabko:08pqaml} a weaker form of convergence 
of probabilities (in particular, convergence of expected average KL divergence) is obtained under 
weaker assumptions.

{\bf The results.} First,  we show that if there is a predictor
that performs well for every measure coming from a class $\C$ of processes, then a predictor can also be obtained as a convex combination $\sum_{k\in\N} w_k\mu_k$ for some $\mu_k\in\C$ and
some $w_k>0$, $k\in\N$. 
This holds if the prediction quality is measured by either total variation distance, or expected average KL divergence: 
one measure of performance that is very strong, the other rather weak. 
The analysis for the total variation case 
relies on the fact that if $\rho$ predicts $\mu$ in total variation distance, then $\mu$ is absolutely continuous
with respect to $\rho$, so that $\rho(x_{1..n})/\mu(x_{1..n})$ converges to a positive number with $\mu$-probability 1
and with a positive $\rho$-probability. However, if we settle for a weaker measure of performance, such as  expected average KL divergence, measures $\mu\in\C$ are typically singular with 
respect to a predictor $\rho$. Nevertheless, since $\rho$ predicts $\mu$ we can show that  $\rho(x_{1..n})/\mu(x_{1..n})$ decreases subexponentially
with $n$ (with high probability or in expectation); then we can use this ratio as an analogue of the density for each time step $n$, and 
find a convex combination of countably many measures from $\C$ that has  desired
predictive properties for each $n$. Combining these predictors for all $n$  results in a predictor that predicts every $\mu\in\C$ in average KL divergence.  The proof techniques developed  have a potential
to be used in solving other questions concerning sequence prediction, in particular, the general question of how to find a predictor
for an arbitrary class $\C$ of measures.

We then  exhibit some sufficient conditions on the class
 $\C$, under which a predictor for all measures in $\C$ exists. It is important to note that none of these conditions relies on a parametrization of any kind.
The conditions presented are of 
two types: conditions on asymptotic behaviour of measures in $\C$, and on their
local (restricted to first $n$ observations) behaviour. Conditions of the first type
concern separability of $\C$ with respect to the total variation distance and the expected average KL divergence.
We show that in the case of total variation separability is a necessary and sufficient condition for the existence of a predictor,
whereas in the case of expected average KL divergence it is sufficient but is not necessary.

The conditions of the second kind concern the ``capacity'' of the sets $\C^n:=\{\mu^n:\mu\in\C\}$, $n\in\N$,
where $\mu^n$ is the measure $\mu$ restricted to the first $n$ observations.
Intuitively, if $\C^n$ is small (in some sense), then prediction
is possible. We measure the capacity of $\C^n$ in two ways. The first way is 
to find the maximum probability given to each sequence $x_1,\dots,x_n$
by some measure in the class, and then take a sum over $x_1,\dots,x_n$.
Denoting the obtained  quantity  $c_n$, one can show that  it grows polynomially in $n$ for 
some important classes of processes, such as i.i.d. or Markov processes.
We show that, in general, if $c_n$ grows subexponentially then a predictor
exists that predicts any measure in  $\C$ in expected average KL divergence.
On the other hand, exponentially growing $c_n$ are not sufficient for prediction. 
 A more refined way to measure the capacity of $\C^n$
is using a concept of channel capacity from information  theory,
which was developed for a closely related problem of finding optimal 
codes for a class of sources. We extend corresponding results from
information theory to show that sublinear growth of channel capacity 
is sufficient for the existence of a predictor, in the sense of expected average divergence.
Moreover, the obtained bounds on the divergence are optimal up to an additive logarithmic term.

The rest of the paper is organized as follows. Section~\ref{s:pre}  introduces the notation and definitions.
In Section~\ref{s:ba} we show that if any predictor works than there is a Bayesian one that works, 
while in Section~\ref{s:ch} we provide several characterizations of predictable classes of processes.
 Section~\ref{s:sep} is concerned with separability, while Section~\ref{s:loc} analyzes conditions
based on local behaviour of measures. Finally, Section~\ref{s:disc} provides outlook and discussion.

As running examples that illustrate the results of each section  we use countable classes of measures, the family of all Bernoulli i.i.d. processes and that of 
all stationary processes.
\section{Preliminaries}\label{s:pre}
Let $\X$ be a finite set. The notation $x_{1..n}$ is used for $x_1,\dots,x_n$. 
 We consider  stochastic processes (probability measures) on $(\X^\infty,\mathcal F)$ where $\mathcal F$
is the sigma-field generated by the cylinder sets  $[x_{1..n}]$, $x_i\in\X, n\in\N$, 
where $[x_{1..n}]$ is the set of all infinite sequences that start with $x_{1..n}$.
For a  finite set $A$ denote $|A|$ its cardinality.
We use  $\E_\mu$ for
expectation with respect to a measure $\mu$.

Next we introduce the measures of the quality of prediction used in this paper.
For two measures  $\mu$ and $\rho$  we are interested in how different 
 the $\mu$- and $\rho$-conditional probabilities are, given a data sample $x_{1..n}$.
Introduce the {\em (conditional) total variation} distance 
$$
v(\mu,\rho,x_{1..n}):= \sup_{A\in\mathcal F} |\mu(A|x_{1..n})-\rho(A|x_{1..n})|.
$$
\begin{definition}
We say that $\rho$ predicts $\mu$ in total variation if 
$$
v(\mu,\rho,x_{1..n})\to0\ \mu\-as
$$
\end{definition}
This convergence is rather strong. In particular, it means that $\rho$-conditional probabilities
of arbitrary far-off events converge to $\mu$-conditional probabilities. 
Moreover,  $\rho$ predicts $\mu$ in
total variation if \cite{Blackwell:62} and only if \cite{Kalai:94} $\mu$ is absolutely continuous with respect to $\rho$:
\begin{theorem}[\cite{Blackwell:62, Kalai:94}]\label{th:bd}
If  $\rho$, $\mu$ are arbitrary probability measures on $(\X^\infty,\mathcal F)$, then $\rho$ predicts $\mu$ in total variation if and only if  $\mu$ is absolutely continuous with respect to $\rho$.
\end{theorem}

Thus, for a class $\C$ of measures there is a predictor $\rho$ that predicts every $\mu\in\C$ in total
variation if and only if every $\mu\in\C$ has a density with respect to $\rho$.
Although such  sets of processes are rather large, they do not include even such basic 
examples as the set of all Bernoulli i.i.d. processes.
That is, there is no $\rho$ that would predict in total variation every Bernoulli i.i.d. process measure $\delta_p$, $p\in[0,1]$,
where $p$ is the probability of $0$. 
Therefore, perhaps for many (if not most) practical applications this measure of the quality of prediction is too strong,
and one is interested in weaker measures of performance.

For two measures $\mu$ and $\rho$ introduce the {\em expected cumulative Kullback-Leibler divergence (KL divergence)} as
\begin{equation}\label{eq:akl} 
  d_n(\mu,\rho):=  \E_\mu
  \sum_{t=1}^n  \sum_{a\in\X} \mu(x_{t}=a|x_{1..t-1}) \log \frac{\mu(x_{t}=a|x_{1..t-1})}{\rho(x_{t}=a|x_{1..t-1})},
\end{equation}
In words, we take the expected (over data) average (over time) KL divergence between $\mu$- and $\rho$-conditional (on the past data) 
probability distributions of the next outcome.
\begin{definition}
We say that $\rho$ predicts $\mu$ in expected average KL divergence if 
$$
{1\over n} d_n(\mu,\rho)\to0.
$$
\end{definition}
This measure of performance is much weaker, in the sense that it requires good predictions only one step ahead, and not on every step
but only on average; also the convergence is not with probability 1 but in expectation. With prediction quality so measured, 
predictors  exist for relatively large
classes of measures; most notably, \cite{BRyabko:88} provides a predictor which predicts every stationary 
process in expected average KL divergence. 
A simple but useful identity that we will need (in the context of sequence prediction introduced also in \cite{BRyabko:88})
is the following
\begin{equation}\label{eq:kl}
 d_n(\mu,\rho)=-\sum_{x_{1..n}\in\X^n}\mu(x_{1..n}) \log \frac{\rho(x_{1..n})}{\mu(x_{1..n})},
\end{equation}
where on the right-hand side we have simply the KL divergence between measures $\mu$ and $\rho$ restricted to the first $n$ observations.

Thus, the results of this work will be established with respect to two very different measures of prediction quality,
one of which is very strong and the other rather weak. This suggests that the facts established reflect some fundamental
properties of the problem of prediction, rather than those pertinent to particular measures of performance. On the other hand,
it remains open to extend the results below to different measures of performance.

\section{Fully nonparametric Bayes predictors}\label{s:ba}
In this section we show that if there is a predictor that predicts every $\mu$ in some class $\C$, then there is a Bayesian 
mixture of countably many elements from $\C$ that predicts every $\mu\in\C$ too. This is established for the 
two notions of prediction quality that were introduced: total variation and expected average KL divergence.
After the theorems we present some examples of families of measures for which predictors exist.

\begin{theorem}\label{th:1} Let $\C$ be a set of probability measures on $(\X^\infty, \mathcal F)$. If there is a measure $\rho$ such that $\rho$ predicts every $\mu\in\C$ in 
total variation, then there is a sequence $\mu_k\in\C$, $k\in\N$ such that the measure $\nu:=\sum_{k\in\N} w_k\mu_k$ predicts every $\mu\in\C$ in 
total variation, where $w_k$ are any positive weights  that sum to 1.
\end{theorem}
This relatively simple fact can be proven in different ways, relying on the mentioned equivalence  \cite{Blackwell:62, Kalai:94} 
of the statements ``$\rho$ predicts $\mu$ in total variation distance'' and ``$\mu$ is absolutely continuous
with respect to $\rho$.'' The proof presented below is not the shortest possible, but it uses ideas and techniques that are then generalized to 
the case of  prediction in expected average KL-divergence, which is more involved, since in all interesting cases 
all measures $\mu\in\C$ are singular with respect to any predictor that predicts all of them. 
Another proof of Theorem~\ref{th:1} can be obtained from Theorem~\ref{th:sep1} in the next section. Yet another way would
be to derive it from algebraic properties of the relation of absolute continuity, given in \cite{Plesner:46}.
\begin{proof}
We break the (relatively easy) proof of this theorem into 3 steps, which will make the 
proof of the next theorem more understandable.

\noindent{\em Step 1: densities.}
 For any $\mu\in\C$, since $\rho$ predicts $\mu$ in total variation, by Theorem~\ref{th:bd}, $\mu$ has a density (Radon-Nikodym derivative) $f_\mu$ with respect 
to $\rho$. Thus, for the set $T_\mu$ of all sequences $x_1,x_2,...\in\X^\infty$ on which $f_\mu(x_{1,2,\dots})>0$ 
(the limit $\lim_{n\rightarrow\infty}\frac {\rho(x_{1..n})}{\mu(x_{1..n})}$ 
exists and is finite and positive) we have $\mu(T_\mu)=1$ and $\rho(T_\mu)>0$. Next we will construct a sequence of measures $\mu_k\in\C$, $k\in\N$ such that 
the union of the sets $T_{\mu_k}$ has probability 1 with respect to every $\mu\in\C$, and will show that this is a sequence of measures whose existence is asserted in the theorem statement.

{\em Step 2: a countable cover and the resulting predictor.}
Let $\epsilon_k:=2^{-k}$ and let $m_1:=\sup_{\mu\in\C}\rho(T_\mu)$. Clearly, $m_1>0$. Find any $\mu_1\in\C$ such that $\rho(T_{\mu_1})\ge m_1-\epsilon_1$, and let
$T_1=T_{\mu_1}$. For $k>1$ define $m_k:=\sup_{\mu\in\C}\rho(T_\mu\backslash T_{k-1})$. If $m_k=0$ then define $T_{k}:=T_{k-1}$, otherwise find any $\mu_k$ such 
that $\rho(T_{\mu_k}\backslash T_{k-1})\ge m_k-\epsilon_k$, and let $T_k:=T_{k-1}\cup T_{\mu_k}$. 
Define the predictor $\nu$ as $\nu:=\sum_{k\in\N}w_k\mu_k$.

{\em Step 3: $\nu$ predicts every $\mu\in\C$.}
Since the sets 
$T_1$, $T_2\backslash T_1,\dots, T_k\backslash T_{k-1},\dots$ are disjoint, 
we must have $\rho(T_{k}\backslash T_{k-1})\to0$, so that $m_k\to0$ (since $m_k\le\rho(T_{k}\backslash T_{k-1})+\epsilon_k\to0$).
Let 
$$
T:=\cup_{k\in\N} T_k.
$$ 
Fix any $\mu\in\C$. 
Suppose that  $\mu(T_{\mu} \backslash T)>0$. Since $\mu$ is absolutely continuous
with respect to $\rho$, we must have $\delta:=\rho(T_{\mu}\backslash T)>0$. Then for every $k>1$ we have
$$m_k=\sup_{\mu'\in\C}\rho(T_{\mu'}\backslash T_{k-1})\ge  \rho(T_{\mu}\backslash T_{k-1})\ge\rho(T_{\mu}\backslash T)=\delta>0,$$ which contradicts 
$m_k\rightarrow0$. Thus, we have shown that 
\begin{equation}\label{eq:mt}
\mu(T\cap T_{\mu})=1.
\end{equation}

Let us show that every $\mu\in\C$ is absolutely continuous with respect to $\nu$.
Indeed, fix any $\mu\in\C$ and suppose $\mu(A)>0$ for some $A\in\mathcal F$.
Then from~(\ref{eq:mt}) we have $\mu(A\cap T)>0$, and, by absolute continuity of $\mu$ with respect to $\rho$,
also $\rho(A\cap T)>0$. Since $T=\cup_{k\in\N}T_{k}$ we must have $\rho(A\cap T_k)>0$ for some $k\in\N$.
Since on the set $T_k$ the measure $\mu_k$ has non-zero density $f_{\mu_k}$ with respect to $\rho$, we must have $\mu_k(A\cap T_k)>0$.
(Indeed, $\mu_k(A\cap T_k)=\int_{A\cap T_k}f_{\mu_k}d\rho>0$.)
Hence,  
$$
\nu(A\cap T_k)\ge w_k \mu_k(A\cap T_k)> 0,
$$ so that $\nu(A)>0$.
Thus, $\mu$ is absolutely continuous with respect to $\nu$, and so, by Theorem~\ref{th:bd}, $\nu$ predicts $\mu$ in total variation distance.
\end{proof}

Thus, examples of families $\C$ for which there is a $\rho$ that predicts every $\mu\in\C$ in total variation, are
limited to families of measures which have a density with respect to some measure $\rho$.  On the one hand, from statistical 
point of view,  such families
are rather large: the assumption that the probabilistic law in question has a density with respect to some (nice) measure 
is  a standard one in statistics. It should also be mentioned that  such families can easily be uncountable. 
On the other hand, even such basic examples as the set of all Bernoulli i.i.d.\ measures does not allow for a predictor 
that predicts every measure in total variation. Indeed, all these processes are singular with respect to one another; in particular, 
each of the non-overlapping sets $T_p$ of all sequences which have limiting fraction $p$ of 0s has probability 1 with respect to one of the measures and 
0 with respect to all others; since there are uncountably many of these measures, there is no measure $\rho$ with respect to which they all would have a density 
(since such a measure should have $\rho(T_p)>0$ for all $p$) .
As it was mentioned,  predicting in total variation distance means predicting with arbitrarily growing horizon  \cite{Kalai:94}, 
while prediction in expected average KL divergence is only concerned with  the probabilities of  the next observation, and
only on time and data average. For the latter measure of prediction quality, consistent predictors exist not only for the class of all Bernoulli 
processes, but also for  the class of all stationary processes \cite{BRyabko:88}. The next theorem establishes the result similar to Theorem~\ref{th:1} 
for expected average KL divergence.

\begin{theorem}\label{th:2} Let $\C$ be a set of probability measures on $(\X^\infty,\mathcal F)$. If there is a measure $\rho$ such that $\rho$ predicts every $\mu\in\C$ in 
expected average KL divergence, then there exist a sequence $\mu_k\in\C$, $k\in\N$ and a sequence $w_k>0,k\in\N$,  such that $\sum_{k,\in\N} w_k=1$, 
and the measure $\nu:=\sum_{k\in\N} w_k\mu_k$ predicts every $\mu\in\C$ in expected average KL divergence.
\end{theorem}
A difference worth noting with respect to the formulation of Theorem~\ref{th:1} (apart from a different measure of divergence) is in that 
in the latter the weights $w_k$ can be chosen arbitrarily, while in Theorem~\ref{th:2} this is not the case.
In general, the statement ``$\sum_{k\in\N} w_k\nu_k$ predicts $\mu$ in expected average KL divergence for some choice of $w_k$, $k\in\N$''  does not imply
``$\sum_{k\in\N} w'_k\nu_k$ predicts $\mu$ in expected average KL divergence for every  summable sequence of positive $w_k', k\in\N$,'' while
the implication trivially holds true if the expected average KL divergence is replaced by the total variation. This is illustrated in 
the last example of this section.
An interesting related question (which is beyond the scope of this paper) is how to chose the weights to optimize the behaviour of a predictor before asymptotic.

The idea of the proof of Theorem~\ref{th:2} is as follows. For every $\mu$ and every $n$ we consider the sets $T_\mu^n$ of those $x_{1..n}$ on which $\mu$ is greater 
than $\rho$. These sets have to have (from some $n$ on) a high  probability with respect to $\mu$. Then since $\rho$ predicts $\mu$ in 
expected average KL divergence, the $\rho$-probability of these sets cannot decrease exponentially fast (that is, it has to be quite large).
(The sequences  $\mu(x_{1..n})/\rho(x_{1..n})$, $n\in\N$ will play the role of densities of the proof of Theorem~\ref{th:1}, and the sets $T_\mu^n$
the role of sets $T_\mu$ on which the density is non-zero.)
We then use, for each given $n$, the same scheme to cover the set $\X^n$ with countably many $T_\mu^n$, as  was used in the proof of Theorem~\ref{th:1} 
to construct a countable covering of the set $\X^\infty$ , obtaining for each $n$ a predictor
$\nu_n$. Then the predictor $\nu$ is obtained as $\sum_{n\in\N} w_n \nu_n$, where the weights decrease subexponentially.
The latter fact ensures that, although the weights depend on $n$, they still play no role  asymptotically.
The technically most involved part of the proof is to show that the sets $T_\mu^n$ in asymptotic have sufficiently 
large weights in those
countable covers that we construct for each $n$. This is used to  demonstrate  the implication ``if a set has a high $\mu$ probability,
then its  $\rho$-probability does not decrease too fast, provided some regularity conditions.''
The proof is broken into the same steps as the (simpler) proof of Theorem~\ref{th:1}, to make the analogy explicit and the proof more understandable.
\begin{proof}
Define the weights $w_k:=wk^{-2}$, where $w$ is the normalizer $6/\pi^2$.

\noindent{\em Step 1: densities.} 
Define the sets 
\begin{equation}\label{eq:t}
T_\mu^n:=\left\{x_{1..n}\in \X^n: \mu(x_{1..n})\ge{1\over n}\rho(x_{1..n})\right\}.
\end{equation} 
Using Markov's inequality, we derive 
\begin{equation}\label{eq:mark}
\mu(\X^n\backslash T_\mu^n) 
 = \mu \left(\frac {\rho(x_{1..n})}{\mu(x_{1..n})} > n\right)\le {1\over n} E_\mu \frac {\rho(x_{1..n})}{\mu(x_{1..n})}={1\over n},
\end{equation}
so that $\mu(T_\mu^n)\to 1$.
(Note that  if $\mu$ is singular with respect to $\rho$, as is typically the case,  then $\frac{\rho(x_{1..n})}{\mu(x_{1..n})}$ converges to 0 $\mu$-a.e. and one
can replace ${1\over n}$ in~(\ref{eq:t}) by 1, while still having $\mu(T_\mu^n)\to1$.)

{\em Step 2n: a countable cover, time $n$.}
Fix an $n\in\N$. Define $m^n_1:=\max_{\mu\in\C}\rho(T_\mu^n)$ (since $\X^n$ are finite all suprema are reached). 
 Find any $\mu^n_1$ such that $\rho^n_1(T_{\mu^n_1}^n)=m^n_1$ and let
$T^n_1:=T^n_{\mu^n_1}$. For $k>1$, let $m^n_k:=\max_{\mu\in\C}\rho(T_\mu^n\backslash T^n_{k-1})$. If $m^n_k>0$, let $\mu^n_k$ be any $\mu\in\C$ such 
that $\rho(T_{\mu^n_k}^n\backslash T^n_{k-1})=m^n_k$, and let $T^n_k:=T^n_{k-1}\cup T^n_{\mu^n_k}$; otherwise let $T_k^n:=T_{k-1}^n$. Observe that 
(for each $n$) there is only a finite number of positive $m_k^n$,
since the set $\X^n$ is finite; let $K_n$ be the largest index $k$ such that $m_k^n>0$. Let 
\begin{equation}\label{eq:nun}
\nu_n:=\sum_{k=1}^{K_n} w_k\mu^n_k.
\end{equation}
As a result of this construction, for every $n\in\N$ every $k\le K_n$ and every  $x_{1..n}\in T^n_k$ 
using~(\ref{eq:t}) we obtain
\begin{equation}\label{eq:ext}
\nu_n(x_{1..n})\ge w_k{1\over n}\rho(x_{1..n}).
\end{equation}

{\em Step 2: the resulting predictor.}
Finally, define 
\begin{equation}\label{eq:nu}
\nu:={1\over 2}\gamma+{1\over2}\sum_{n\in\N}w_n\nu_n,
\end{equation}
 where $\gamma$ is the i.i.d. measure with equal probabilities of all $x\in\X$ 
(that is, $\gamma(x_{1..n})=|\X|^{-n}$ for every $n\in\N$ and every $x_{1..n}\in\X^n$). 
We will show that $\nu$  predicts every $\mu\in\C$, and 
then in the end of the proof (Step~r) we will show how to replace $\gamma$ by a combination of a countable set of elements of $\C$ (in fact, $\gamma$ 
is just a regularizer which ensures that $\nu$-probability of any word is never too close to~0). 

{\em Step 3: $\nu$ predicts every $\mu\in\C$.}
Fix any $\mu\in\C$. 
Introduce the parameters $\epsilon_\mu^n\in(0,1)$, $n\in\N$, to be defined later, and let $j_\mu^n:=1/\epsilon_\mu^n$.
Observe that $\rho(T^n_k\backslash T^n_{k-1})\ge \rho(T^n_{k+1}\backslash T^n_k)$, for any $k>1$ and any $n\in\N$, by definition of these sets.
Since the sets $T^n_k\backslash T^n_{k-1}$, $k\in\N$ are disjoint, we obtain $\rho(T^n_k\backslash T^n_{k-1})\le 1/k$. Hence,  $\rho(T_\mu^n\backslash T_j^n)\le \epsilon_\mu^n$ for some $j\le j_\mu^n$,
since  otherwise $m^n_j=\max_{\mu\in\C}\rho(T_\mu^n\backslash T^n_{j_\mu^n})> \epsilon_\mu^n$ so that  $\rho(T_{j_\mu^n+1}^n\backslash T^n_{j_\mu^n}) > \epsilon_\mu^n=1/j_\mu^n$, which is a contradiction. 
Thus,   
\begin{equation}\label{eq:tm}
\rho(T_\mu^n\backslash T_{j_\mu^n}^n)\le \epsilon_\mu^n.
\end{equation}
We can upper-bound $\mu(T_\mu^n\backslash T^n_{j^n_\mu})$ as follows. 
First, observe that
\begin{multline}\label{eq:mut}
d_n(\mu,\rho) 
=   -\sum_{x_{1..n}\in T^n_\mu\cap T^n_{j^n_\mu}}\mu(x_{1..n})\log\frac{\rho(x_{1..n})}{\mu(x_{1..n})} 
\\
-\sum_{x_{1..n}\in T^n_\mu\backslash  T^n_{j^n_\mu}}\mu(x_{1..n})\log\frac{\rho(x_{1..n})}{\mu(x_{1..n})} \\- \sum_{x_{1..n}\in \X^n\backslash T^n_\mu}\mu(x_{1..n})\log\frac{\rho(x_{1..n})}{\mu(x_{1..n})}
\\
=
I+II+III.
\end{multline}
Then, from~(\ref{eq:t}) we get 
\begin{equation}\label{eq:e1}
I\ge -\log n.
\end{equation}
Observe that for every $n\in\N$ and every set $A\subset \X^n$, using Jensen's inequality we can obtain
\begin{multline}\label{eq:jen}
-\sum_{x_{1..n}\in A}\mu(x_{1..n})\log\frac{\rho(x_{1..n})}{\mu(x_{1..n})}
=  -\mu(A)\sum_{x_{1..n}\in A}{1\over\mu(A)}\mu(x_{1..n})\log\frac{\rho(x_{1..n})}{\mu(x_{1..n})}
\\
\ge -\mu(A)\log{\rho(A)\over\mu(A)} \ge -\mu(A)\log\rho(A) -{1\over2}. 
\end{multline}
Thus, from~(\ref{eq:jen}) and~(\ref{eq:tm})
we get 
\begin{equation}\label{eq:e2}
II
\ge  -\mu(T_\mu^n\backslash T^n_{j^n_\mu}) \log\rho(T_\mu^n\backslash T^n_{j^n_\mu})- 1/2
\ge -\mu(T_\mu^n\backslash T^n_{j^n_\mu}) \log \epsilon_\mu^n - 1/2.
\end{equation}
Furthermore,  
\begin{multline}\label{eq:e3}
III\ge \sum_{x_{1..n}\in \X^n\backslash T^n_\mu}\mu(x_{1..n})\log\mu(x_{1..n}) 
\ge \mu(\X^n\backslash T^n_\mu)\log\frac{\mu(\X^n\backslash T^n_\mu)}{|\X^n\backslash T^n_\mu|}\\\ge -{1\over2} - \mu(\X^n\backslash T^n_\mu)n\log|\X|\ge-{1\over2} -\log|\X|,
\end{multline} 
where in the second inequality we have used the fact that entropy is maximized when all events are equiprobable, in the third one we used $|\X^n\backslash T^n_\mu|\le|\X|^n$,
while the last inequality follows from~(\ref{eq:mark}).
Combining~(\ref{eq:mut}) with the bounds~(\ref{eq:e1}), (\ref{eq:e2}) and~(\ref{eq:e3})  we obtain 
\begin{equation*}
d_n(\mu,\rho) \ge -\log n  -\mu(T_\mu^n\backslash T^n_{j^n_\mu}) \log \epsilon_\mu^n  - 1 - 
  \log|\X|,
\end{equation*}
so that
\begin{equation}\label{eq:mu2}
 \mu(T_\mu^n\backslash T^n_{j^n_\mu}) \le {1\over-\log \epsilon_\mu^n}\Big(d_n(\mu,\rho) +\log n  +1 + 
     \log|\X|\Big).
\end{equation}
Since $d_n(\mu,\rho)=o(n)$, 
 we can define the parameters $\epsilon^n_\mu$ in such a way that $-\log \epsilon^n_\mu=o(n)$ while
at the same time  the bound~(\ref{eq:mu2}) gives $\mu(T_\mu^n\backslash T^n_{j^n_\mu})=o(1)$. Fix such a choice of $\epsilon^n_\mu$.
Then,   using $\mu(T^n_\mu)\to1$,  we can conclude
\begin{equation}\label{eq:xt}
\mu(\X^n\backslash T^n_{j^n_\mu})\le \mu(\X^n\backslash T^n_{\mu})+ \mu(T^n_{\mu}\backslash T^n_{j^n_\mu}) =o(1).
\end{equation}

We proceed with the proof of $d_n(\mu,\nu)=o(n)$. 
For any $x_{1..n}\in T^n_{j_\mu^n}$  we have
\begin{equation}\label{eq:i}
\nu(x_{1..n})\ge {1\over 2}w_n\nu_n(x_{1..n})
\ge{1\over 2}w_n w_{j_{\mu}^n} {1\over n}\rho(x_{1..n})=\frac{w_nw}{2n}(\epsilon_\mu^n)^2\rho(x_{1..n}),
\end{equation}
where the first inequality follows from~(\ref{eq:nu}), the second from~(\ref{eq:ext}), and in the equality we have used $w_{j_{\mu}^n}=w/(j_{\mu}^n)^2$
and  $j_\mu^n=1/\epsilon^\mu_n$.
 Next we use the decomposition
\begin{equation}\label{eq:12}
d_n(\mu,\nu)= -\sum_{x_{1..n}\in T^n_{j_\mu^n}}\mu(x_{1..n})\log\frac{\nu(x_{1..n})}{\mu(x_{1..n})} \\ 
- \sum_{x_{1..n}\in \X^n\backslash T^n_{j_\mu^n}}\mu(x_{1..n})\log\frac{\nu(x_{1..n})}{\mu(x_{1..n})}  = I + II.
\end{equation}
From~(\ref{eq:i})  we find 
\begin{multline}\label{eq:1}
I\le -\log\left(\frac{w_nw}{2n}(\epsilon_\mu^n)^2\right)  - \sum_{x_{1..n}\in T^n_{j_\mu^n}}\mu(x_{1..n})\log\frac{\rho(x_{1..n})}{\mu(x_{1..n})}\\
=
(1+3\log n - 2\log\epsilon_\mu^n-2\log w) +\left(d_n(\mu,\rho)+ \sum_{x_{1..n}\in\X^n\backslash T^n_{j_\mu^n}}\mu(x_{1..n})\log\frac{\rho(x_{1..n})}{\mu(x_{1..n})}\right)
\\
\le o(n) -  \sum_{x_{1..n}\in\X^n\backslash T^n_{j_\mu^n}}\mu(x_{1..n})\log\mu(x_{1..n})\\ 
\le o(n)+\mu(\X^n\backslash T^n_{j_\mu^n})n\log|\X|=o(n),
\end{multline}
where in the second inequality we have used $-\log\epsilon_\mu^n=o(n)$ and $d_n(\mu,\rho)=o(n)$, in the last inequality we have again used the fact that the entropy is maximized when all events are equiprobable, 
while the last equality follows from~(\ref{eq:xt}). 
Moreover, from~(\ref{eq:nu}) we find
\begin{equation}\label{eq:2}
II\le \log 2 - \sum_{x_{1..n}\in\X^n\backslash  T^n_{j_\mu^n}}\mu(x_{1..n})\log\frac{\gamma(x_{1..n})}{\mu(x_{1..n})}
\le 1 +n\mu(\X^n\backslash T^n_{j_\mu^n})\log|\X|=o(n),
\end{equation}
where in the last inequality we have used $\gamma(x_{1..n})=|\X|^{-n}$ and $\mu(x_{1..n})\le 1$, and the last equality follows from~(\ref{eq:xt}).

From~(\ref{eq:12}), (\ref{eq:1}) and~(\ref{eq:2}) we conclude ${1\over n}d_n(\nu,\mu)\to0$.

{\em Step r: the regularizer $\gamma$}. It remains to show that the  i.i.d. regularizer $\gamma$ in the definition of $\nu$~(\ref{eq:nu}), can be replaced by a convex combination of a countably many elements from $\C$.
Indeed, for each $n\in\N$, denote
$$
A_n:=\{x_{1..n}\in \X^n: \exists\mu\in\C\ \mu(x_{1..n})\ne0\},
$$ and let for each $x_{1..n}\in \X^n$ the measure $\mu_{x_{1..n}}$ be any measure from $\C$ such that $\mu_{x_{1..n}}(x_{1..n})\ge{1\over2}\sup_{\mu\in\C}\mu(x_{1..n})$.
Define 
$$
 \gamma_n'(x'_{1..n}):={1\over |A_n|}\sum_{x_{1..n}\in A_n}\mu_{x_{1..n}}(x'_{1..n}),
$$ for each
$x'_{1..n}\in A^n$, $n\in\N$, and let  $\gamma':=\sum_{k\in\N}w_k\gamma'_k$. 
For every $\mu\in\C$ we have 
$$
\gamma'(x_{1..n})\ge w_n|A_n|^{-1} \mu_{x_{1..n}}(x_{1..n})\ge{1\over2} w_n |\X|^{-n} \mu(x_{1..n})
$$ for
every $n\in\N$ and every $x_{1..n}\in A_n$, which clearly suffices to establish the bound $II=o(n)$ as in~(\ref{eq:2}).
\end{proof}

{\noindent \bf Example: countable classes} of measures. A very simple but  rich  example of a class $\C$ that satisfies
the conditions of both the theorems above, is any countable family $\C=\{\mu_k: k\in\N\}$ of measures. 
In this case, any mixture predictor $\rho:=\sum_{k\in\N} w_k\mu_k$ predicts all $\mu\in\C$ both in total variation 
and in expected average KL divergence. A particular instance that has gained much attention in the literature is the 
family of all computable measures. Although countable, this family of processes is rather rich. The problem 
of predicting all computable measures was introduced in \cite{Solomonoff:78} where a mixture predictor was proposed.

{\noindent \bf Example: Bernoulli i.i.d.\ processes.} Consider the class $\C_B=\{\mu_p: p\in[0,1]\}$ of 
all Bernoulli i.i.d.\ processes: $\mu_p(x_k=0)=p$ independently for all $k\in\N$. Clearly, this family is uncountable.
Moreover, each set $$T_p:=\{x\in \X^\infty:\text{ the limiting fraction of 0s in $x$ equals }p\},$$ 
 has probability 1 with respect 
to $\mu_p$ and probability 0 with respect to any $\mu_{p'}:p'\ne p$. Since the sets $T_p$, $p\in[0,1]$ are non-overlapping, there is no measure $\rho$ for which $\rho(T_p)>0$ for all $p\in[0,1]$.
That is, there is no measure $\rho$ with respect to which all $\mu_p$ are absolutely continuous.
Therefore, by Theorem~\ref{th:bd}, a predictor that predicts any $\mu\in\C_B$  in total variation does not exist, demonstrating that this notion
of prediction is rather strong. However, we know (e.g. \cite{Krichevsky:93}) that the Laplace predictor~(\ref{eq:lapl}) predicts
every Bernoulli i.i.d. process in expected average KL divergence (and not only).
Hence, Theorem~\ref{th:1} implies that there is a countable mixture predictor for this family too.
Let us find such a predictor. Let $\mu_q:q\in Q$ be the family of all Bernoulli i.i.d. measures with rational probability of 0, 
and let $\rho:=\sum_{q\in Q} w_q\mu_q$, where $w_q$ are arbitrary positive weights that sum to 1. Let $\mu_p$ be  any Bernoulli i.i.d.
process. Let $h(p,q)$ denote the divergence $p\log(p/q)+(1-p)\log(1-p/1-q)$. For each $\epsilon$ we can find a $q\in Q$ such that $h(p,q)<\epsilon$. Then 
\begin{multline}\label{eq:bern}
 {1\over n} d_n(\mu_p,\rho)={1\over n}\E_{\mu_p}\log\frac{\log\mu_p(x_{1..n})}{\log\rho(x_{1..n})}\le {1\over n}\E_{\mu_p}\log\frac{\log\mu_p(x_{1..n})}{w_q\log\mu_q(x_{1..n})} 
\\=-\frac{\log w_q}{n} +  h(p,q)\le \epsilon+ o(1). 
\end{multline}
Since this holds for each $\epsilon$ we conclude that ${1\over n} d_n(\mu_p,\rho)\to0$ and $\rho$ predicts every $\mu\in\C_B$ in expected average KL divergence.

{\noindent \bf Example: stationary processes.} In \cite{BRyabko:88}  a predictor $\rho_R$ was constructed which predicts every stationary 
process $\rho\in\C_S$ in expected average KL divergence. (This predictor is obtained as a mixture of
predictors for $k$-order Markov sources, for all $k\in\N$.) Therefore, Theorem~\ref{th:2} implies that there is also a countable mixture predictor
for this family of processes. Such a predictor can be constructed as follows (the proof in this example is based on the proof in  \cite{BRyabko:06}, Appendix~1).
 Observe that the family $\C_k$ of
 $k$-order stationary binary-valued Markov
processes is parametrized by $2^k$ $[0,1]$-valued parameters: probability of observing 0 after observing $x_{1..k}$, for each $x_{1..k}\in\X^k$.
For each $k\in\N$ let $\mu^k_q$, $q\in Q^{2^k}$ be the (countable) family of all stationary $k$-order Markov processes with rational 
values of all the parameters. We will show that any predictor $\nu:=\sum_{k\in\N}\sum_{q\in Q^{2^k}} w_k w_q \mu^k_q$, where $w_k$, $k\in\N$ and
$w_q, q\in Q^{2^k}$, $k\in\N$ are any sequences of positive real weights that sum to 1, predicts every stationary $\mu\in\C_S$ in expected average 
KL divergence. 
For $\mu\in\C_S$ and $k\in\N$ define the $k$-order conditional Shannon entropy $h_k(\mu):=\E_\mu\log\mu(x_{k+1}|x_{1..k})$.
We have $h_{k+1}(\mu)\ge h_k(\mu)$ for every $k\in\N$ and $\mu\in\C_S$, and the limit 
\begin{equation}\label{eq:hlim}
h_\infty(\mu):=\lim_{k\to\infty}h_k(\mu)
\end{equation}
is called the limit Shannon entropy, see e.g. \cite{Gallager:68}.
Fix some $\mu\in\C_S$. It is easy to see  that for  every $\epsilon>0$ and every $k\in\N$ we can find a $k$-order stationary
 Markov measure $\mu^k_{q_\epsilon}$, 
$q_\epsilon\in Q^{2^k}$ with rational values of the parameters, such that 
\begin{equation}\label{eq:muq}
\E_\mu\log\frac{\mu(x_{k+1}|x_{1..k})}{\mu^k_{q_\epsilon}(x_{k+1}|x_{1..k})}<\epsilon.
\end{equation}
We have 
\begin{multline}\label{eq:must}
 {1\over n} d_n(\mu,\nu) \le -\frac{\log w_k w_{q_\epsilon}}{n} + {1\over n} d_n(\mu,\mu^k_{q_\epsilon})\\ 
 =O(k/n) + {1\over n} \E_\mu\log\mu(x_{1..n}) - {1\over n} \E_\mu\log\mu^k_{q_\epsilon}(x_{1..n}) \\ 
= o(1)+ h_\infty(\mu) - {1\over n} \E_\mu\sum_{k=1}^n\log\mu^k_{q_\epsilon}(x_t|x_{1..t-1})
 \\ =  o(1)+ h_\infty(\mu) -  {1\over n} \E_\mu\sum_{t=1}^{k}\log\mu^k_{q_\epsilon}(x_t|x_{1..t-1})  - \frac{n-k}{n}\E_\mu\log\mu^k_{q_\epsilon}(x_{k+1}|x_{1..k})
\\ \le o(1)+ h_\infty(\mu)- \frac{n-k}{n}(h_k(\mu)- \epsilon),
\end{multline}
where the first inequality is derived analogously to~(\ref{eq:bern}), the first equality follows from~(\ref{eq:kl}), the second
equality follows from the Shannon-McMillan-Breiman theorem (e.g. \cite{Gallager:68}), that states that ${1\over n}\log\mu(x_{1..n})\to h_\infty(\mu)$
  in expectation (and a.s.)
 for every $\mu\in\C_S$, and~(\ref{eq:kl}); in the third equality we have used the fact that $\mu^k_{q_\epsilon}$ is $k$-order Markov
and $\mu$ is stationary, whereas the last inequality follows from~(\ref{eq:muq}). Finally, since the choice of $k$ and $\epsilon$ was arbitrary, 
from~(\ref{eq:must}) and~(\ref{eq:hlim}) we
obtain $\lim_{n\to\infty}{1\over n} d_n(\mu,\nu)=0$.

{\noindent \bf Example: weights may matter.} Finally, we provide an example that illustrates the difference between the formulations of Theorems~\ref{th:1} and~\ref{th:2}:
in the latter the weights are not arbitrary. We will construct a sequence of measures $\nu_k, k\in\N$, a measure $\mu$, and two sequences of positive weights $w_k$ and $w_k'$
with $\sum_{k\in\N} w_k= \sum_{k\in\N} w_k'=1$, for which $\nu:=\sum_{k\in\N}w_k\nu_k$ predicts $\mu$ in expected average KL divergence, but $\nu':=\sum_{k\in\N}w'_k\nu_k$
does not. Let 
$\nu_k$ be a deterministic measure that first outputs $k$ 0s and then only 1s, $k\in\N$.  
Let $w_k=w/k^2$ with $w=6/\pi^2$ and $w_k'=2^{-k}$. Finally, let $\mu$ be a deterministic measure that outputs only 0s. 
We have $d_n(\mu,\nu)=- \log (\sum_{k\ge n} w_k)=O(\log n)$, but $d_n(\mu,\nu')=- \log (\sum_{k\ge n} w'_k)=-\log (2^{-n+1})=n-1\ne o(n)$,
proving the claim.

\section{Characterizing predictable classes}\label{s:ch}
Knowing that a mixture of a countable subset gives a predictor if there is one, a notion that naturally comes to mind 
when trying to characterize families of processes for which  a predictor exists, is separability.
Can we say that there is a predictor for a class $\C$ of measures if and only if $\C$ is separable? 
Of course, to talk about separability we need a suitable topology on the space of all measures, or at least on $\C$. If the formulated 
questions were to have a positive answer, we would need  a different topology for each of the notions 
of predictive quality that we consider. Sometimes these measures of predictive quality indeed define  a nice
enough structure of a probability space, but sometimes they do not. The question  whether there exists a 
topology on $\C$, separability with respect to which is equivalent to the existence of a predictor, is already more vague and less appealing.
Nonetheless, in the case of total variation distance we obviously have a candidate topology: that of total variation distance, 
and indeed separability with respect to this topology is equivalent to the existence of a predictor, as the next 
theorem shows. This theorem also implies Theorem~\ref{th:1}, thereby providing an alternative proof for the latter.
In the case of expected average KL divergence the situation is different. While one can introduce a topology based on it, 
separability with respect to this topology turns out to be a sufficient but not a necessary condition for the existence of a predictor, 
as is shown in Theorem~\ref{th:sep2}.
\subsection{Separability}\label{s:sep}
\begin{definition}[unconditional total variation distance]
 Introduce the  (unconditional) total variation distance 
$$
v(\mu,\rho):= \sup_{A\in\mathcal F} |\mu(A)-\rho(A)|.
$$
\end{definition}

\begin{theorem}\label{th:sep1}
 Let $\C$ be a set of probability measures on $(\X^\infty,\mathcal F)$. There is a measure $\rho$ such that $\rho$ predicts every $\mu\in\C$ in 
total variation if and only if $\mathcal C$ is separable with respect to the topology of total variation distance.
In this case any measure $\nu$ of the form $\nu=\sum_{k=1}^\infty w_k\mu_k$, where $\{\mu_k: k\in\N\}$ is any dense countable subset of $\C$ and $w_k$ are any positive
weights that sum to 1,  predicts every $\mu\in\C$ in total variation.
\end{theorem}
\begin{proof}
{\em Sufficiency and the mixture predictor.} Let $\C$ be separable in total variation distance, and let $\mathcal D=\{\nu_k:k\in\N\}$ be its dense countable subset.
We have to show that $\nu:=\sum_{k\in\N}w_k\nu_k$, where $w_k$ are any positive real weights that sum to 1, predicts every $\mu\in\C$ in total variation.
To do this, it is enough to show that $\mu(A)>0$ implies $\nu(A)>0$ for every $A\in\mathcal F$ and every $\mu\in\C$. Indeed, 
let $A$ be such that $\mu(A)=\epsilon>0$. Since $\mathcal D$ is dense in $\C$, there is a $k\in\N$ such that $v(\mu,\nu_k)<\epsilon/2$.
Hence $\nu_k(A)\ge\mu(A)-v(\mu,\nu_k)\ge \epsilon/2$ and $\nu(A)\ge w_k\nu_k(A)\ge w_k\epsilon/2>0$.

{\em Necessity.}
 For any $\mu\in\C$, since $\rho$ predicts $\mu$ in total variation, $\mu$ has a density (Radon-Nikodym derivative) $f_\mu$ with respect 
to $\rho$.
% Thus, for the set $T:=\{x\in X^\infty: \exists\mu\in\C\ f_\mu(x)\ne0\}$ we have $\mu(T)=1$ for all $\mu\in\C$.
We can define $L_1$ distance with respect to $\rho$ as follows $L_1^\rho(\mu,\nu)=\int_{\X^\infty}|f_{\mu}-f_\nu|d\rho$.
The set of all measures that have a density with respect to $\rho$ is separable with respect to this distance 
(for example a dense countable subset can be constructed based on measures whose densities are step-functions with finitely many steps,   that take only rational values, see e.g. \cite{Kolmogorov:75});
therefore, its subset $\C$ is also separable. Let $\mathcal D$ be any dense countable  subset of $\C$.
Thus, for every $\mu\in\C$ and every $\epsilon$ there is a $\mu'\in \mathcal D$ such that $L_1^\rho(\mu,\mu')<\epsilon$.
For every measurable set  $A$ we have 
$$
|\mu(A)-\mu'(A)|=\left|\int_A f_\mu d\rho -\int_A f_{\mu'}d\rho\right|\le \int_A|f_\mu-f_{\mu'}|d\rho\le\int_{\X^\infty}|f_\mu-f_{\mu'}|d\rho<\epsilon.
$$
Therefore, $v(\mu,\mu')=\sup_{A\in\mathcal F}|\mu(A)-\mu'(A)|<\epsilon$ and the set $\C$ is separable in total variation distance.
\end{proof}

\begin{definition}[asymptotic KL ``distance'' $D$] 
Define  asymptotic expected average KL divergence between measures $\mu$ and $\rho$ as 
\begin{equation}\label{eq:akl2}
 D(\mu,\rho)=\limsup_{n\rightarrow\infty} {1\over n}  d_n(\mu,\rho).
\end{equation}
\end{definition}

\begin{theorem}\label{th:sep2}
For a set $\C$ of measures, separability with respect to the asymptotic expected average KL divergence $D$ is a sufficient but not a  necessary condition
for the existence of a predictor:
\begin{itemize}
 \item[(i)] If there exists a countable set $\mathcal D:=\{\nu_k:k\in\N\}\subset\C$ such that  for every $\mu\in\C$ and every $\epsilon>0$ there is a measure 
$\mu'\in\mathcal D$ such  that $D(\mu,\mu')<\epsilon$, then every measure  $\nu$ of the form $\nu=\sum_{k=1}^\infty w_k\mu_k$, where  $w_k$ are any positive
weights that sum to 1,  predicts every $\mu\in\C$ in expected average KL divergence.
\item[(ii)]  There is an uncountable set $\C$ of measures and a measure $\nu$ such that $\nu$ predicts every $\mu\in\C$ in expected average KL divergence,
but $\mu_1\ne\mu_2$ implies $D(\mu_1,\mu_2)=\infty$ for every $\mu_1,\mu_2\in\C$; in particular, $\C$ is not separable with respect to $D$.
\end{itemize}
\end{theorem}
\begin{proof}
 {\em (i)} Fix $\mu\in\C$. For every $\epsilon>0$ pick $k\in\N$ such that $D(\mu, \nu_k)<\epsilon$. We have 
$$ 
 d_n(\mu,\nu) = \E_\mu\log\frac{\mu(x_{1..n})}{\nu({x_{1..n}})} \le \E_\mu\log\frac{\mu(x_{1..n})}{w_k\nu_k{(x_{1..n}})}  = -\log w_k + d_n(\mu,\nu_k) \le n\epsilon + o(n).
$$ Since this holds for every $\epsilon$, we conclude ${1\over n} d_n(\mu,\nu)\to0$.

{\em (ii)} Let $\C$ be the set of all deterministic sequences (measures concentrated on just one sequence) such that the number of 0s in the 
first $n$ symbols is less than $\sqrt{n}$. Clearly, this set is uncountable. It is easy to check that $\mu_1\ne\mu_2$ implies $D(\mu_1,\mu_2)=\infty$ for every $\mu_1,\mu_2\in\C$, but 
the predictor $\nu$ given by $\nu(x_n=0)=1/n$ independently for different $n$, predicts every $\mu\in\C$ in expected average KL divergence.
\end{proof}

{\noindent\bf Examples.} Basically, the examples of the preceding section carry over here. Indeed, the example of countable families
is trivially also an example of separable (with respect to either of the considered topologies) family.
For Bernoulli i.i.d. and $k$-order Markov processes, the (countable) sets of processes that have rational 
values of the parameters, considered in the previous section, are dense both in the topology of the parametrization
and with respect to the asymptotic average divergence $D$. It is also easy to check from the arguments presented in the corresponding example
of Section~\ref{s:ba} that the family of all $k$-order stationary Markov processes with rational values
of the parameters, where we take all $k\in\N$, is dense with respect to $D$ in the set $\C_S$ of all stationary processes, so
that $\C_S$ is separable with respect to $D$. Thus, the sufficient but not necessary condition of separability is satisfied in this case.
On the other hand, neither of these latter families is separable with respect to the topology of total variation distance.

\subsection{Conditions based on the local behaviour of measures.} \label{s:loc}
Next we provide some sufficient conditions for the existence of a predictor based 
on local characteristics of the class of measures, that is, measures truncated to the first $n$ observations.
First of all, it must be noted that necessary and sufficient conditions cannot be obtained this way. The basic
example is that of a family $\C_0$ of all deterministic sequences that are 0 from some time on. This is a countable
class of measures which is very easy to predict. Yet the class of measures on $\X^n$ obtained by truncating
all measures in $\C_0$ to the first $n$ observation coincides with what would be obtained by truncating all deterministic
measures to the first $n$ observation, the latter class being obviously not predictable at all (see also examples below). 
Nevertheless, considering this kind of local behaviour of measures, one can obtain not only sufficient conditions 
for the existence of a predictor, but also rates of convergence of the prediction error. It also gives some
ideas of how to construct  predictors, for the cases when the sufficient conditions obtained are met.

For  a class   $\C$ of  stochastic processes   and a sequence $x_{1..n}\in\X^n$ introduce
the coefficients
\begin{equation}\label{eq:sup}
 c_{x_{1..n}}(\C):=\sup_{\mu\in\C}\mu(x_{1..n}).
\end{equation}
Define also the normalizer 
\begin{equation}\label{eq:nor}
 c_n(\C):=\sum_{x_{1..n}\in\X^n}c_{x_{1..n}}(\C).
\end{equation}
\begin{definition}[NML estimate]
  The  normalized maximum likelihood (e.g. \cite{Krichevsky:93}) estimator $\lambda$ is defined as 
\begin{equation}\label{eq:nml}
\lambda_\C(x_{1..n}):= \frac{1}{c_n(\C) } c_{x_{1..n}}(\C),
\end{equation}
for each $x_{1..n}\in\X^n$.
\end{definition}
The family $\lambda_\C(x_{1..n})$ (indexed by $n$) in general does  not immediately define a stochastic process over $\X^\infty$ ($\lambda_\C$ are not consistent for different $n$);
thus, in particular, using average KL divergence for measuring prediction quality would not make sense, since
$$d_n(\mu(\cdot|x_{1..n-1}),\lambda_\C(\cdot|x_{1..n-1}))$$ can be negative, as the following example shows.

{\noindent\bf Example: negative $d_n$ for NML estimates}. Let the processes $\mu_i$, $i\in\{1,\dots,4\}$ be defined on
the steps $n=1,2$ as follows. $\mu_1(00)=\mu_2(01)=\mu_4(11)=1$, while $\mu_3(01)=\mu_3(00)=1/2$. 
We have $\lambda_\C(1)=\lambda_\C(0)=1/2$, while  $\lambda_\C(00)=\lambda_\C(01)=\lambda_\C(11)=1/3$. If we 
define $\lambda_\C(x|y)=\lambda_\C(yx)/\lambda_\C(y)$ we get $\lambda_\C(1|0)=\lambda_\C(0|0)=2/3$. Then
$d_2(\mu_3(\cdot|0),\lambda_\C(\cdot|0))=\log3/4<0$.

Yet, by taking an appropriate mixture, it is still possible to construct a predictor (a stochastic process) based on $\lambda$, 
that predicts all the measures in the class. 
\begin{definition}[predictor $\rho_c$]\label{def:nml}
Let $w:=6/\pi^2$  and let $w_k:=\frac{1}{w k^2}$.
Define a measure $\mu_k$ as follows. On the first $k$ steps it is defined as $\lambda_\C$, and 
for $n>k$ it outputs only zeros with probability 1; so, $\mu_k(x_{1..k})=\lambda_\C(x_{1..k})$ and 
$\mu_k(x_n=0)=1$ for $n>k$. 
Define the measure $\rho_c$ as
\begin{equation}\label{eq:r2}
\rho_c=\sum_{k=1}^\infty w_k\mu_k.
\end{equation}
\end{definition}
Thus, we have taken the normalized maximum likelihood estimates $\lambda_n$ for each $n$ and continued
them arbitrarily (actually, by a deterministic sequence) to obtain a sequence of measures on $(\X^\infty, \mathcal F)$ that can be summed.

\begin{theorem}\label{th:ml} For  a class $\C$ of stochastic processes the predictor  $\rho_c$ defined above satisfies
\begin{equation}\label{eq:mlb}
{1\over n} d_n(\mu,\rho_c)\le \frac{\log c_n(\C)}{n} + O\left(\frac{\log n}{n}\right);
\end{equation}
in particular, if \begin{equation}\label{eq:cond1}
\log c_n(\C) =o(n).
\end{equation}
then $\rho_c$ predicts every $\mu\in\C$ in expected average KL divergence.
\end{theorem}
\begin{proof}
Indeed, 
\begin{multline}\label{eq:mlproof}
 {1\over n} d_n(\mu,\rho_c) 
  = \frac{1}{n}\E \log \frac {\mu(x_{1..n})}{\rho_c(x_{1..n})} 
  \le \frac{1}{n} \E \log \frac {\mu(x_{1..n})}{w_n \mu_n(x_{1..n})}\\
  \le \frac{1}{n} \log\frac{c_n(\C)}{w_n} = \frac{1}{n}(\log c_n(\C) + 2\log n + \log w).
\end{multline}
\end{proof}

{\noindent\bf Example: i.i.d., finite-memory.} To illustrate the applicability of the theorem we first consider 
the class of i.i.d. processes $\C_B$ over the binary alphabet $\X=\{0,1\}$. 
It is easy to see that for each $x_1,\dots,x_n$ 
$$
\sup_{\mu\in\C_B} \mu(x_{1..n})=(k/n)^k(1-k/n)^{n-k}
$$
where $k=\#\{i\le n: x_i=0\}$ is the number of 0s in $x_1,\dots,x_n$. 
For the constants $c_n(\C)$ we can derive
\begin{multline*}
 c_n(C)=\sum_{x_{1..n}\in \X^n}\sup_{\mu\in\C_B}\mu(x_{1..n})=\sum_{x_{1..n}\in \X^n} (k/n)^k(1-k/n)^{n-k}\\=\sum_{k=0}^n{n\choose k}(k/n)^k(1-k/n)^{n-k}\le 
\sum_{k=0}^n\sum_{t=0}^n{n\choose k}(k/n)^t(1-k/n)^{n-t}=n+1,
\end{multline*}
so that  $c_n(C)\le n+1$.

In general, for the class $\C_k$ of {\bf processes with memory $k$} over a finite
space $\X$ we can get polynomial $c_n(\C)$  (see e.g. \cite{Krichevsky:93}, and also \cite{Ryabko:07pqisit}).
Thus, with respect to  finite-memory processes, the conditions 
of Theorem~\ref{th:ml} leave ample space for the growth of $c_n(\C)$, since~(\ref{eq:cond1}) allows  subexponential growth of $c_n(\C)$.
Moreover, these conditions are tight, as the following example shows.

{\noindent \bf Example: exponential coefficients are not sufficient.} 
Observe that the condition~(\ref{eq:cond1}) cannot 
be relaxed further, in the sense that exponential coefficients $c_n$ are 
not sufficient for prediction. Indeed, for the class of all deterministic 
processes (that is,  each process from the class produces some fixed sequence 
of observations with probability 1) we have $c_n=2^n$, while obviously 
for this class a predictor does not exist. 

{\noindent \bf Example: stationary processes.} For the set of all stationary processes we can obtain $c_n(C)\ge2^n/n$ (as is easy to 
see by considering periodic $n$-order Markov processes, for each $n\in\N$), so
that the conditions of Theorem~\ref{th:ml} are not satisfied.  This cannot be fixed, since  
uniform rates of convergence cannot be obtained for this family of processes, as was shown
in \cite{BRyabko:88}.

{\bf Optimal rates of convergence}. A natural question that arises with respect to the bound~(\ref{eq:mlb}) is whether it can be
matched by a lower bound. This question is closely related to the optimality
of the normalized maximum likelihood estimates used in the construction of the predictor. In general,
since NML estimates are not optimal, neither are the rates of convergence in~(\ref{eq:mlb}).
To obtain (close to) optimal rates one has to consider a different measure of capacity.
 
To do so, we make the following connection to a problem 
in information theory. Let $\mathcal P(\X^\infty)$ be the set of all stochastic processes (probability measures) on
the space $(\X^\infty,\mathcal F)$, and let $\mathcal P(\X)$ be the set of probability distributions over a (finite) set $\X$. For a class $\C$ of measures we are interested in a predictor
that has a small (or minimal) worst-case (with respect to the class $\C$) probability of error.
Thus, we are interested in the quantity
\begin{equation}\label{eq:infsup}
\inf_{\rho\in\mathcal P(\X^\infty)} \sup_{\mu\in\C} D(\mu,\rho),
\end{equation}
where the infimum is taken over all stochastic processes $\rho$, and $D$ is the 
asymptotic expected average KL divergence~(\ref{eq:akl2}).
(In particular, we are interested in the conditions under which the quantity~(\ref{eq:infsup})
equals zero.) This problem has been studied for the case when the probability measures
are over a finite set $\X$, and $D$ is replaced simply by the KL divergence $d$ between 
the measures. 
Thus, the problem was to find the probability measure $\rho$ (if it exists) on which the following 
minimax is attained 
\begin{equation}\label{eq:infsup2}
R(A):=\inf_{\rho\in\mathcal P(\X)}\sup_{\mu\in A} d(\mu,\rho),
\end{equation}
where $A\subset\mathcal P(\X)$.
This problem is closely related to the problem of finding the best code for the class of sources
$A$, which was its original motivation. 
The normalized maximum likelihood distribution considered above does not in general
lead to the optimum solution for this problem. 
The optimum solution is obtained through the  result that relates the minimax~(\ref{eq:infsup2})  to the so-called channel capacity.
\begin{definition}[Channel capacity]
For a set $A$ of measures on a finite set $\X$ the {\em channel capacity} of
$A$ is defined as 
\begin{equation}\label{eq:cc}
 C(A):=\sup_{P\in \mathcal P_0(A)} \sum_{\mu\in S(P)} P(\mu) d(\mu,\rho_P),
\end{equation} where $\mathcal P_0(A)$ is the set of all probability distributions on $A$ that have a finite support, $S(P)$ is the (finite) support
 of a distribution $P\in\mathcal P_0(A)$, and
$\rho_P=\sum_{\mu\in S(P)} P(\mu)\mu$.
\end{definition}
 It is shown in \cite{BRyabko:79,Gallager:76} that 
$
C(A)=R(A),
$
thus reducing the problem of finding a minimax to an optimization problem.
For probability measures over infinite spaces this result ($R(A)=C(A)$)
was generalized by \cite{Haussler:97}, but the divergence between probability distributions is measured
by KL divergence (and not asymptotic average KL divergence), which gives  infinite $R(A)$ e.g.  already for the class of i.i.d. processes. 

However, truncating measures in a class $\C$ to the first $n$ observations, we can
use the results about channel capacity to analyze the predictive properties of the class.
Moreover, the rates of convergence that can be obtained along these lines are close to optimal.
In order to pass from measures minimizing the divergence for each individual $n$ to a process that minimizes
the divergence for all $n$ we use the same idea as when constructing the process~$\rho_c$.

\begin{theorem}\label{th:cc} Let $\C$ be a set of measures on $(\X^\infty,\mathcal F)$, and let $\C^n$ be the class
of measures from $\C$ restricted to $\X^n$. 
There exists a measure $\rho_C$ such that
\begin{equation}\label{eq:ccb}
 {1\over n} d_n(\mu,\rho_C)\le \frac{ C(\C^n)}{n} + O\left(\frac{\log n}{n}\right);
\end{equation}
in particular, if  $C(\C^n)/n\rightarrow0$ then  $\rho_C$ predicts every $\mu\in\C$ in expected average KL divergence.
Moreover, for any measure $\rho_C$ and every $\epsilon>0$ there exists $\mu\in\C$ such that
$$
 {1\over n} d_n(\mu,\rho_C)\ge \frac{ C(\C^n)}{n} -\epsilon.
$$
\end{theorem}
\begin{proof} 
As shown in \cite{Gallager:76}, for each $n$ there exists a  sequence $\nu^n_k$, $k\in\N$
of measures on $\X^n$ such that 
$$
 \lim_{k\rightarrow\infty} \sup_{\mu\in \C^n} d_n(\mu,\nu^n_k)\rightarrow C(\C^n).
$$ 
For each $n\in\N$ find an index $k_n$ such that  
$$ 
|\sup_{\mu\in \C^n} d_n(\mu,\nu^n_{k_n}) - C(\C^n)|\le 1.
$$
Define the measure $\rho_n$ as follows. On the first $n$ symbols it coincides with $\nu^n_{k_n}$ and
$\rho_n(x_m=0)=1$ for $m>n$. Finally, set $\rho_C=\sum_{n=1}^\infty w_n\rho_n$, where $w_k=\frac{1}{w n^2}, w=6/\pi^2$.
We have to show that $\lim_{n\rightarrow\infty}{1\over n} d_n(\mu,\rho_C)=0$ for every $\mu\in\C$.
Indeed, similarly to~(\ref{eq:mlproof}), we have
\begin{multline}\label{eq:ccproof}
{1\over n} d_n(\mu,\rho_C)=\frac{1}{n}\E_\mu\log\frac{\mu(x_{1..n})}{\rho_C(x_{1..n})} \\
 \le \frac{\log w_k^{-1}}{n} +\frac{1}{n} \E_\mu\log\frac{\mu(x_{1..n})}{\rho_n(x_{1..n})} 
 \le
 \frac{\log w +2\log n}{ n}+\frac{1}{n} d_n(\mu,\rho_n)\\
\le o(1) + \frac{C(\C^n)}{n}.
\end{multline} 

The second statement follows from the fact \cite{BRyabko:79,Gallager:76} that \ $C(\C^n)=R(\C^n)$ \ (cf.~(\ref{eq:infsup2})).
\end{proof}  

Thus, if the channel capacity  $C(\C^n)$ grows sublinearly, a predictor 
can be constructed for the class of processes $\C$. In this case 
the problem of constructing the predictor is reduced to finding the channel 
capacities for different $n$ and finding the corresponding measures on which 
they are  attained or approached. 

\noindent {\bf Examples.} For the class of all Bernoulli i.i.d. processes,
the channel  capacity   $C(\C^n_B)$   is known to be $O(\log n)$ \cite{Krichevsky:93}.
For the family of all stationary processes it is $O(n)$, so that the conditions of Theorem~\ref{th:cc} are satisfied
for the former but not for the latter. 

We also remark that the requirement of a sublinear channel capacity cannot be relaxed,
in the sense that a linear channel capacity is not sufficient for prediction, since
it is the maximal possible  capacity for a set of measures on $\X^n$, achieved e.g. on the set of all measures, or on 
the set of all deterministic sequences.

\section{Discussion}\label{s:disc}
The first possible  extension of the results of the paper that comes to mind is to find out 
whether the same  holds for other measures of performance, such as prediction in KL divergence
without time-averaging,  or with probability 1 rather then in expectation, or with respect to other
measures of prediction error, such as absolute distance.  (See    \cite{Ryabko:07pqisit} 
for a discussion of  different measures of performance and
relations between them.)
Maybe the same 
results can be obtained in more general formulations, for example, using $f$-divergences of \cite{Csiszar:67}.

More generally, the questions we addressed in this work are a part of a larger problem:
given an arbitrary class $\C$ of stochastic processes, find the best predictor for it. 
We have considered two subproblems: first, in which form to look for a predictor if one exists.
Here we have shown that if any predictor works then a Bayesian one works  too.
The second one is to characterize families of processes for which a predictor exists.
Here we have analyzed what the notion of separability furnishes in this respect, as well 
as identified some simple sufficient conditions based on the local behaviour of measures in the class.
 Another approach would be to 
identify the conditions which two measures $\mu$ and $\rho$ have to satisfy in order for $\rho$ to predict $\mu$.
For prediction in total variation such conditions have been identified \cite{Blackwell:62,Kalai:94} and, in particular, in the context of the present
 work, they turn out to be very useful. \cite{Kalai:94} also provide some characterization for the case of a weaker notion 
of prediction: difference between conditional probabilities of the next (several) outcomes (weak merging of opinions).
In \cite{Ryabko:08pqaml} some sufficient conditions are found for the case of prediction in expected average KL divergence,
and prediction in average KL divergence with probability~1. 
Of course, another  very natural approach to the general  problem posed above
is to try and find predictors (in the form of algorithms) for some particular classes of processes which 
are of practical interest.
Towards this end, we have found a rather simple  form that some solution to this question has if a solution exists: 
a Bayesian predictor whose prior is concentrated on a countable set. We have also identified some sufficient
conditions under which a predictor can actually be constructed (e.g. using NML estimates). 
However, the larger question of how to construct an optimal predictor for an arbitrary given family of processes, remains open.

Taking an even more general perspective, one can consider the problem of finding the best response
to the actions of a (stochastic) environment, which itself responds to the actions of a learner. 
Allowing into consideration environments that change their behaviour in response to the action of the learner,
clearly makes the problem much more difficult, but it also dramatically extends the range of applications.
For this general problem one can pose the same questions: given a set $\C$ of environments, how can we construct a
learner that is (asymptotically) optimal if any environment from $\C$ is chosen to generate the data? One can 
 consider Bayesian learners for this formulation too \cite{Hutter:04uaibook}; it would be interesting to find out
whether one can show that when there is an learner which is optimal in every environment from $\C$, then there is
a Bayesian learner with a countably supported prior that has this property too. 

%\subsection*{Acknowledgements}
%
% \subsection*{Acknowledgements} The author is grateful to Marcus Hutter for some fruitful discussions on the subject, 
% in particular on the conference version \citep{Ryabko09:pq3} of Theorems~\ref{th:1} and~\ref{th:2}.% results of Section~\ref{s:ba}.

%\bibliography{../../my}

\end{document}